\documentclass[journal,twoside,web]{ieeecolor}
\usepackage{generic}
\usepackage{cite}
\usepackage{amsmath,amssymb,amsfonts}
\usepackage{algorithmic}
\usepackage{graphicx}
\usepackage{textcomp}

\usepackage{booktabs}
\usepackage{multirow}
\usepackage{url}
\usepackage{bbm}
\usepackage{subfigure}
	\usepackage{graphicx}
	\usepackage{algorithm}
	\usepackage{algorithmic}
	\usepackage{rotating}
	\newtheorem{theorem}{Theorem}

	\newtheorem{definition}{Definition}

\def\BibTeX{{\rm B\kern-.05em{\sc i\kern-.025em b}\kern-.08em
    T\kern-.1667em\lower.7ex\hbox{E}\kern-.125emX}}
\begin{document}
\title{Single-Edge Node Injection Threats to GNN-Based Security Monitoring in Industrial Graph Systems}
\author{Wenjie Liang,  Ranhui Yan, Jia Cai, and You-Gan Wang
	\thanks{This work was supported in part by the National Natural Science Foundation of China under Grant 12271111. Wenjie Liang and Ranhui Yan contributed equally to this work. The corresponding authors are Jia Cai and You-Gan Wang.}
	\thanks{Wenjie Liang  is with School of Digital Technology, Guangdong Finance $\&$ Trade Vocational College, Guangzhou, Guangdong, 510320, China. (e-mail: liangwenjie@gdcmxy.edu.cn). }
	\thanks{Ranhui Yan is with School of Information and Intelligent Engineering, Guangzhou Xinhua University, Guangzhou, Guangdong, 510520, China. (e-mail: yanranhui1373@xhsysu.edu.cn). }
	\thanks{Jia Cai and You-Gan Wang are with School of Statistics and Data Science, Guangdong University  of Finance $\&$ Economics, Guangzhou, Guangdong, 510320, China. (e-mails: jiacai1999@gdufe.edu.cn (J. Cai), wangyg@gdufe.edu.cn(Y. G. Wang).)}}

\maketitle

\begin{abstract} 
Graph neural networks (GNNs) are increasingly adopted in industrial graph-based monitoring systems (e.g., Industrial internet of things (IIoT)  device graphs, power-grid topology models, and manufacturing communication networks) to support anomaly detection, state estimation, and asset classification. In such settings, an adversary that compromises a small number of edge devices may inject counterfeit nodes (e.g., rogue sensors, virtualized endpoints, or spoofed substations) to bias downstream decisions while evading topology- and homophily-based sanitization. This paper formulates deployment-oriented node-injection attacks under constrained resources and proposes the \emph{Single-Edge Graph Injection Attack} (SEGIA), in which each injected node attaches to the operational graph through a single edge. SEGIA integrates a pruned SGC surrogate, multi-hop neighborhood sampling, and reverse graph convolution--based feature synthesis with a similarity-regularized objective to preserve local homophily and survive edge pruning. Theoretical analysis and extensive evaluations across datasets and defenses show at least $25\%$ higher attack success than representative baselines under substantially smaller edge budgets. These results indicate a system-level risk in industrial GNN deployments and motivate lightweight admission validation and neighborhood-consistency monitoring.
\end{abstract} 


\section{Introduction}
\label{sec:introduction}
\IEEEPARstart{I}{ndustrial} cyber--physical systems (CPS) increasingly rely on graph-based monitoring to support anomaly detection, asset identification, fault localization, and risk triage under stringent latency and reliability constraints. In industrial internet of things (IIoT), smart-grid, and manufacturing environments, operational graphs couple devices, controllers, identities, and cyber--physical dependencies. Consequently, the security and trustworthiness of graph-driven decision pipelines has become a system-level requirement.

Graph neural networks (GNNs) enable industrial informatics by fusing topology and telemetry and propagating information across multi-hop dependencies. This capability aligns with industrial coupling (e.g., cascading grid effects, lateral movement paths in operational technology (OT) networks, and correlated process faults) but also enlarges the attack surface: localized manipulations can influence predictions beyond the immediate neighborhood and bias operational actions. A deployment-relevant threat is the node injection attack, in which an adversary introduces malicious or counterfeit entities into the monitored graph without directly rewriting protected connectivity. In industrial deployments, injection may occur through onboarding a rogue sensor or gateway, instantiating a spoofed protocol endpoint, or creating virtualized ``shadow'' assets. Once admitted into message passing, injected nodes can distort learned representations and induce misclassification (e.g., masking compromised equipment, suppressing alarms, or distorting prioritization), thereby translating model-level errors into operational risk.

Despite growing work on GNN robustness, many adversarial formulations remain misaligned with industrial constraints. Common assumptions include large structural budgets (e.g., multi-edge injections or extensive topology edits) and strong attacker knowledge (e.g., global topology or full features). In contrast, industrial adversaries typically face admission controls, partial observability, and evolving graphs. Moreover, practical monitoring pipelines may apply homophily-oriented edge pruning and neighborhood-consistency checks, increasing the likelihood that conspicuous injections are detected. Finally, prior results are often reported as overall accuracy degradation without explicitly connecting induced errors to system-level risk in CPS and OT operations.

To address these gaps, this paper studies node injection under resource constraints and stealth requirements motivated by industrial deployments. The proposed \emph{Single-Edge Graph Injection Attack} (SEGIA) enforces a deployment-oriented constraint: each injected node connects to the operational graph through a single edge, limiting exposure to topology validation and edge-budget auditing. SEGIA integrates (i) a pruning-aware surrogate to anticipate edge-pruning defenses, (ii) multi-hop neighborhood sampling for scalable optimization under partial knowledge, and (iii) reverse graph convolution--based feature synthesis with similarity regularization to preserve local homophily and reduce detectability. The resulting impact is analyzed from a system-level perspective, yielding implications for lightweight trustworthiness controls.

The main contributions are summarized as follows:
\begin{itemize}
	\item \textbf{Deployment-oriented threat model:} A resource-constrained node-injection threat is formalized for industrial graph-based monitoring in IIoT/CPS/smart-grid environments, capturing admission and connectivity constraints and their operational implications.
	\item \textbf{Single-edge, stealth-aware attack:} \textbf{SEGIA} is proposed as a single-edge injection framework that combines local sampling, reverse feature synthesis, and a pruning-aware surrogate objective to induce high-impact misclassification under strict edge budgets.
	\item \textbf{Evidence under constrained budgets:} Across datasets and defenses, SEGIA achieves at least $25\%$ higher attack success than representative baselines under substantially smaller edge budgets, indicating residual risk under homophily-oriented sanitization in the evaluated settings.
\end{itemize}

The remainder of this paper is organized as follows. Section~\ref{rela} reviews related work. Section~\ref{prel} introduces preliminaries and threat models. Section~\ref{method} presents SEGIA and analysis. Section~\ref{expe} reports experimental results, followed by discussion in Section~\ref{conl}. Proofs are provided in the Appendix.

\section{Related Work} \label{rela} 
This section reviews adversarial attacks on graph neural networks (GNNs), with an emphasis on \emph{graph injection attacks} (GIAs) that align with admission-constrained industrial graph monitoring in IIoT/CPS and smart-grid settings. Prior work has shown that GNN inference can be degraded by adversarial perturbations to graph structure and attributes~\cite{bojchevski2019adversarial,jia2020certified,ma2021graph,AdversarialAttackSurvey}. 

\textbf{Graph modification attacks (GMAs).} GMAs directly edit existing edges and/or node features~\cite{jia2020certified,AdversarialAttackSurvey,Nettack,zugner2018adversarial}. While useful for robustness evaluation, these assumptions may be difficult to satisfy in operational graphs where authenticated connectivity and registered attributes are protected, limiting an attacker's ability to arbitrarily rewrite topology or features~\cite{jia2020certified,Nettack,HAO}. 

\textbf{Graph injection attacks (GIAs).} GIAs add malicious nodes and incident edges without modifying original nodes or links, which better matches permission constraints and device admission processes. Representative methods include AFGSM~\cite{wang2020scalable} (architecture-specific approximations), NIPA~\cite{NIPA} (sequential label/edge generation with higher cost), and TDGIA~\cite{TDGIA} (heuristic injection with feature optimization). Subsequent studies have explored restrictive and stealth-oriented settings: G-NIA~\cite{G-NIA} shows that even single-node injection can substantially degrade performance, while HAO~\cite{HAO} and CANA~\cite{CANA} incorporate homophily-aware constraints and camouflage objectives to improve survivability under detection and pruning. Other variants extend GIAs to class-specific poisoning (NICKI)~\cite{SHARMA2023236}, gradient-free single-node injection (G$^2$-SNIA)~\cite{Chen2023G2SNIA}, dynamic graphs (SFIA)~\cite{JIANG2024100185}, multi-view constructions (MV-RGCN)~\cite{MVRGCN}, generation-based camouflage (IMGIA)~\cite{yangimgia}, text-attributed graphs~\cite{Leitag}, global injection via label propagation (LPGIA)~\cite{ZhuLPGIA2024}, and isolated-subgraph injection (LiSA)~\cite{Zhang2025LiSA}. However, GIAs that are simultaneously \emph{low-footprint} (small edge budgets), \emph{efficient} under partial observability, and \emph{robust to homophily-oriented sanitization} remain underexplored for deployment-oriented industrial graph analytics, which motivates SEGIA.

\section{Preliminaries}
\label{prel}

\subsection{Notation}
Table~\ref{notation} summarizes the main symbols used throughout the paper. Additional notation is defined at first use.

\begin{table*}[!htbp]
	\caption{Summary of notation used throughout the paper.}
	\centering
	\begin{tabular}{ll}
		\toprule
		Symbol & Description \\
		\midrule
		$V, E$ & Node set and edge set of the graph \\
		$N, D$ & Numbers of nodes and feature dimension, respectively \\
		$G=(A,X)$ & Attributed graph with adjacency matrix $A$ and feature matrix $X$ \\
		$G'=(A',X')$ & Attacked graph with adjacency matrix $A'$ and feature matrix $X'$ \\
		$X_I\in\mathbb{R}^{N_I\times D}$ & Feature matrix of injected nodes \\
		$A_I\in\{0,1\}^{N_I\times N}$ & Adjacency between injected and original nodes \\
		$\tilde{A}=A+I$ & Adjacency with self-loops ($I$ is the identity matrix) \\
		$\hat{A}=\tilde{D}^{-1/2}\tilde{A}\tilde{D}^{-1/2}$ & Normalized adjacency ($\tilde{D}$ is the degree matrix of $\tilde{A}$) \\
		$\mathcal{N}(u)$ & Neighbor set of node $u$ \\
		$d_u$ & Degree of node $u$ \\
		$e_{Y_u}$ & One-hot vector for class $Y_u$ \\
		$C$ & Number of classes \\
		$V_L, V_t, V_c$ & Labeled nodes, target nodes, and victim nodes, respectively \\
		$N_I$ & Number of injected nodes \\
		$f_\theta$ & GNN classifier with parameters $\theta$ \\
		$W$ & Trainable weight matrix in the surrogate \\
		$Z_{\mathrm{SGC}}, Z_{\mathrm{PrSGC}}$ & Outputs of SGC and PrSGC surrogate, respectively \\
		${\cal L}_{\mathrm{train}}$ & Training loss (cross-entropy) on $V_L$ \\
		$\Delta$ & Perturbation budget (defined in Sec.~\ref{prel}; SEGIA enforces single-edge injection) \\
		$P\in\{0,1\}^{N\times N}$ & Binary pruning mask in PrSGC \\
		$\varepsilon$ & Similarity threshold used to construct $P$ \\
		$\alpha$ & Weight of similarity regularization \\
		$Pr$ & Perturbation rate \\
		$V_s^k$ & Sampled node set at layer $k$ ($k=1,\ldots,K$; $K$ is sampling depth) \\
		$s^k=|V_s^k|$ & Number of sampled nodes at layer $k$ \\
		$\mathbf{S}(\cdot)$ & One-hop neighbor sampling operator \\
		$M^k\in\mathbb{R}^{s^{k-1}\times s^k}$ & Connectivity matrix between $V_s^{k-1}$ and $V_s^{k}$ \\
		$T$ & Maximum number of iterations \\
		\midrule
		$H^{(2)}$ & Two-layer hidden representation under the surrogate model \\
		\midrule
		$h_u$ & Node-centric homophily score of node $u$ \\
		$r_u$ & Degree-normalized average of neighbors' features for node $u$ \\
		$H_G=\{h_v(G):v\in V\}$ & Multiset of homophily scores on $G$ \\
		$g_\psi$ & Homophily-based defender (edge-pruning function) \\
		$\mathbbm{1}_{\mathrm{con}}(u,v)$ & Indicator of connectivity between $u$ and $v$ \\
		$Z'$ & Output logits after defense/pruning \\
		\bottomrule
	\end{tabular}
	\label{notation}
\end{table*}

\subsection{Surrogate GNN Model}
\label{subsec21}
Let $G=(A,X)$ be an attributed graph with $A\in\{0,1\}^{N\times N}$ and $X\in\mathbb{R}^{N\times D}$. We use Simple Graph Convolution (SGC)~\cite{SGC} as an efficient surrogate:
\begin{equation}
Z_{\mathrm{SGC}}=\mathrm{softmax}\!\left(\hat{A}^2 X W\right),
\label{eq:SGC}
\end{equation}
where $\hat{A}=\tilde{D}^{-1/2}\tilde{A}\tilde{D}^{-1/2}$, $\tilde{A}=A+I$, and $W$ is trainable. Given labeled nodes $V_L\subseteq V$ with labels $Y_u\in\{0,\ldots,C-1\}$, the parameters are trained by minimizing cross-entropy on $V_L$:
\begin{equation*}
{\cal L}_{\mathrm{train}}=\sum_{u\in V_L}\ell\!\left(f_\theta(A,X)_u,\,Y_u\right).
\label{eq:ltrain}
\end{equation*}

\subsection{Graph Adversarial Attacks}
Graph adversarial attacks construct an attacked graph $G'=(A',X')$ under a perturbation budget $\|G'-G\|\le\Delta$ to increase error on a victim set $V_c\subseteq V$. A common bilevel formulation is
\begin{align*}
\min\quad & {\cal L}_{\mathrm{atk}}\!\left(f_{\theta^*}(G')\right) \notag\\
\text{s.t.}\quad & \theta^*=\arg\min_\theta {\cal L}_{\mathrm{train}}\!\left(f_\theta(G_{\mathrm{train}})\right),\quad \|G'-G\|\le\Delta
\label{eq:bilevel}
\end{align*}
with ${\cal L}_{\mathrm{atk}}=-{\cal L}_{\mathrm{train}}$.

\textbf{Graph modification attacks (GMAs).}
GMAs perturb existing edges and/or features. While widely used for robustness evaluation, they may be impractical in operational graphs with protected topology and registered attributes.

\textbf{Graph injection attacks (GIAs).}
GIAs inject $N_I$ nodes without modifying original nodes:
\begin{equation}
X'=\begin{bmatrix}X\\X_I\end{bmatrix},\qquad
A'=\begin{bmatrix}A & A_I^\top\\A_I & \mathbf{0}\end{bmatrix},
\label{eq:gia_def}
\end{equation}
where $X_I\in\mathbb{R}^{N_I\times D}$ are injected features and $A_I\in\{0,1\}^{N_I\times N}$ encodes injected-to-original edges. Following~\cite{HAO}, the GIA objective is
\begin{equation}
\begin{aligned}
\min\quad & {\cal L}_{\mathrm{GIA}}\!\left(f_{\theta^*}(G')\right),\\
\text{s.t.}\quad & \theta^*=\arg\min_\theta {\cal L}_{\mathrm{train}}\!\left(f_\theta(G_{\mathrm{train}})\right),\quad \|G'-G\|\le\Delta.
\end{aligned}
\label{eq:original-attack-loss}
\end{equation}
Here, ${\cal L}_{\mathrm{GIA}}=-{\cal L}_{\mathrm{train}}$  is defined to increase prediction error on the targeted/victim nodes under the injection budget.

These definitions enable a deployment-oriented formulation of single-edge node injection under homophily-aware sanitization. Next, we present SEGIA, including local neighborhood sampling, reverse feature synthesis, and a pruning-aware surrogate objective.

\section{Methodology}\label{method} 

This section presents the proposed SEGIA with an emphasis on \textbf{industrial graph-based systems} and \textbf{deployment feasibility}. In industrial informatics, graphs are commonly constructed from OT and IIoT telemetry to support state awareness and security analytics. Nodes may represent \emph{devices} (PLCs, RTUs, sensors, gateways), \emph{users/identities} (operator accounts, service principals), or \emph{grid assets} (substations, feeders, relays), while edges encode \emph{communication links}, \emph{control dependencies}, \emph{asset-to-asset couplings}, or \emph{trust/authorization relations}. SEGIA models an adversary that introduces a small number of \textbf{counterfeit entities} (fake devices/accounts/virtual assets) and establishes \textbf{minimal connectivity} to influence a GNN-based classifier used in industrial monitoring (e.g., anomaly triage, asset classification, or event categorization), while reducing detectability under topology- and homophily-based sanitization. 
 
\subsection{Industrial threat model and feasibility constraints} 
We consider both \emph{evasion} and \emph{targeted} settings. Given a trained victim GNN, the attacker injects malicious nodes to force misclassification on a set of target nodes $V_t$ (e.g., causing a compromised device to be classified as benign, suppressing a fault indicator, or shifting risk labels used for operator actions). SEGIA is designed under constraints typical in industrial deployments: 
\begin{itemize}
	\item \textbf{Entity realism:} 
	Each injected node corresponds to a plausible industrial entity (e.g., a newly onboarded sensor, a rogue gateway endpoint, a spoofed user/service identity, or a virtualized grid asset record). 
	\item \textbf{Feature fabrication realism:} 
	Injected features must resemble feasible telemetry/metadata (protocol statistics, configuration fingerprints, operational measurements, or embedding-based attributes) and avoid abrupt deviations that trigger consistency checks. \item \textbf{Edge construction constraints:} OT segmentation, admission control, and whitelist policies limit new connectivity. SEGIA enforces a \textbf{single-edge budget} per injected node, reflecting a low-footprint link such as ``talking'' to one gateway, associating with one user group, or attaching to one asset neighborhood. 
	\item \textbf{Limited knowledge and scalability:} 
	The attacker may only observe a local neighborhood around targets or obtain limited query feedback; thus SEGIA uses neighborhood sampling and an efficient surrogate to keep optimization practical. 
\end{itemize} 
 
\begin{figure*}[!htbp] 
	\centering 
	\includegraphics[width=0.7\linewidth]{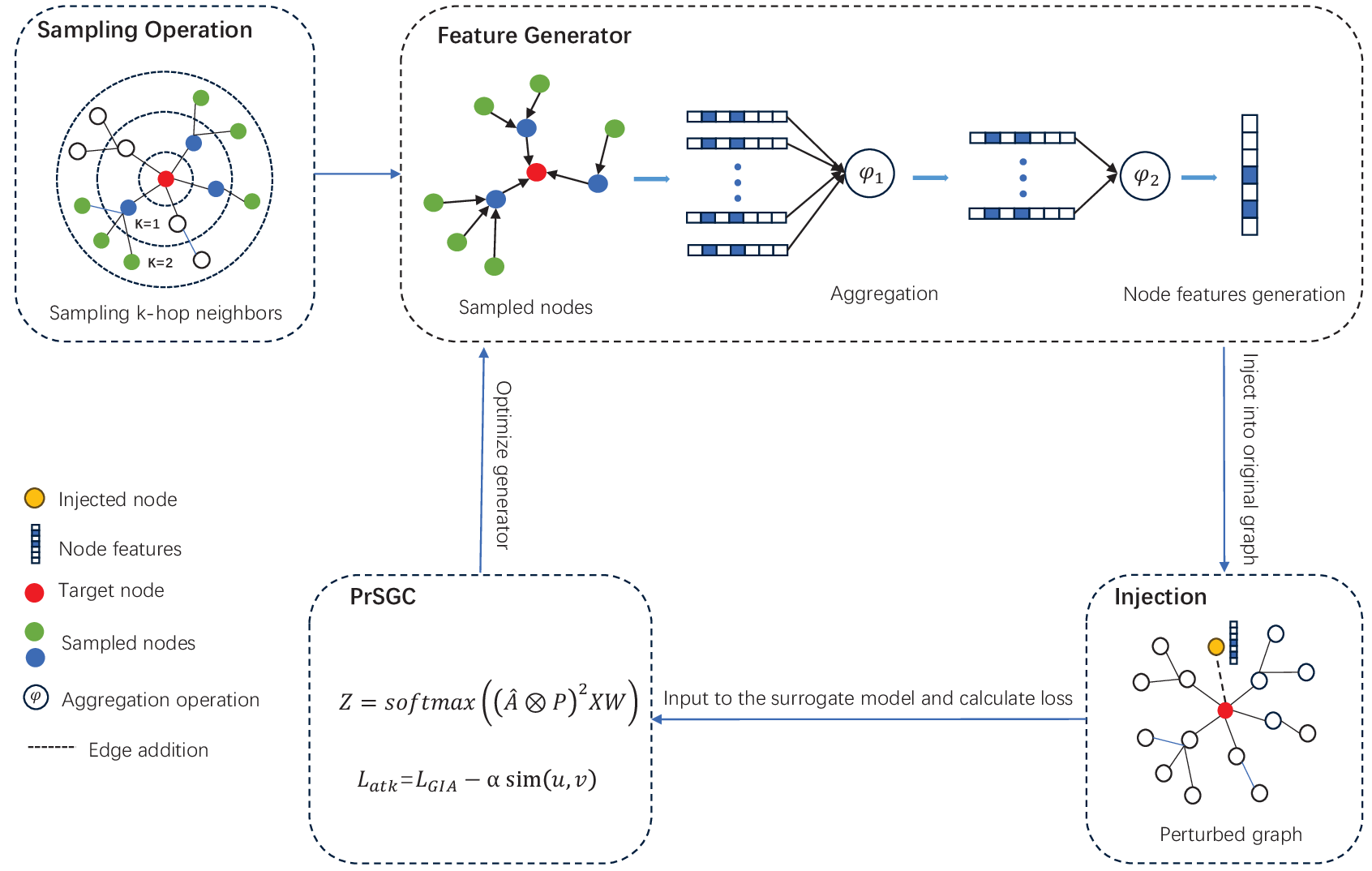} 
	\caption{Workflow of the proposed SEGIA, interpreted as counterfeit-entity injection into industrial graph-based monitoring.} 
	\label{fig:overall-framework} 
\end{figure*} 

\begin{algorithm}[!htbp]
	\caption{SEGIA: Single-Edge Graph Injection Attack (industrial interpretation)}
	\label{alg:SEGIA}
	\begin{algorithmic}[1]
		\REQUIRE Graph $G=(A,X)$; target set $V_t$; sampling depth $K$; iterations $T$; regularization weight $\alpha$; pruning threshold $\varepsilon$ (PrSGC).
		\ENSURE Attacked graph $G'=(A',X')$ with injected nodes and single-edge attachments.
		
		\STATE Sample a $K$-hop neighborhood around targets to obtain $\{V_s^{k}\}_{k=0}^{K}$ using Eq.~\eqref{eq:sample-node-set_industrial}.
		\STATE Initialize feature generator parameters $\{W^{k}\}_{k=1}^{K}$ (and optionally initialize anchor mapping $j(i)$).
		\FOR{$t = 1$ to $T$}
		\FOR{$k = K$ down to $1$}
		\STATE Construct $M^k$ by Eq.~\eqref{eq:M_industrial} and row-normalize $\widetilde M^k$ by Eq.~\eqref{eq:mean-M_industrial}.
		\STATE Update reverse-aggregated features using Eq.~\eqref{eq:feat-agg_industrial}.
		\ENDFOR
		\STATE Set injected feature matrix $X_I \leftarrow X^{0}$; optionally project $X_I$ to valid ranges.
		\STATE Choose one anchor $v_{j(i)} \in V_s^{1}\cup V_t$ for each injected node $u_i$ and construct $A_I$ with exactly one nonzero entry per row.
		\STATE Form the attacked graph by Eq.~\eqref{eq:gia_def}: $X'=[X;X_I]$, $A'=\begin{bmatrix}A & A_I^\top\\A_I & \mathbf{0}\end{bmatrix}$.
		\STATE Evaluate ${\cal L}_{\mathrm{atk}}(G')$ in Eq.~\eqref{eq:atk-loss_industrial} using the PrSGC surrogate in Eq.~\eqref{eq:prsgc_industrial}.
		\STATE Update generator parameters $\{W^{k}\}$ via (projected) gradient descent.
		\ENDFOR
		\RETURN $G'$.
	\end{algorithmic}
\end{algorithm}
Accordingly, SEGIA consists of three components: 
(i) \textbf{local neighborhood sampling} around targets,
(ii) \textbf{reverse graph convolution feature synthesis}, and
(iii) \textbf{optimization with a pruned surrogate and similarity-regularized objective}. The framework is shown in Fig.~\ref{fig:overall-framework} and summarized in Algorithm~\ref{alg:SEGIA}.
\subsection{Local neighborhood sampling} 
\label{subsec31} 
Industrial graphs can be large (enterprise-scale IIoT, utility-scale grid topology, and plant-wide communication graphs). To ensure engineering feasibility, SEGIA restricts optimization to a \emph{local} multi-hop neighborhood around the target set $V_t$. This reflects realistic attacker visibility (nearby devices, adjacent network segments, or reachable asset dependencies) and reduces computation. Let $K$ denote the sampling depth and let $\mathbf{S}(\cdot)$ be a one-hop neighbor sampling operator. Initialize $V_s^{0}=V_t$ and iteratively sample neighbors: 
\begin{equation} 
V_s^{k}=V_s^{k-1}\cup \mathbf{S}\!\left(V_s^{k-1}\right),\qquad k=1,2,\ldots,K. \label{eq:sample-node-set_industrial} 
\end{equation} Thus, $V_s^{k}$ contains nodes within $k$ hops of the targets. In industrial terms, this captures local dependencies (shared gateways, peer controllers, adjacent substations/feeders, or users with shared roles) that typically dominate message passing in GNN inference.

\begin{figure}[!htbp] 
	\centering 
	\includegraphics[width=0.4\linewidth]{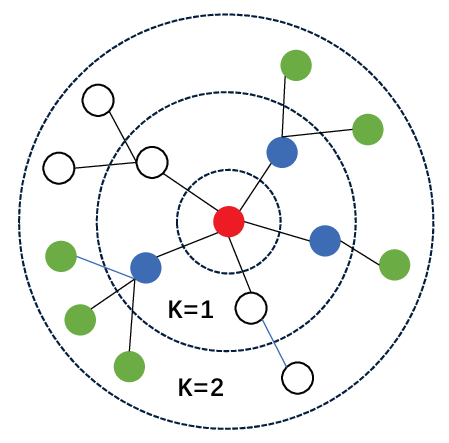} \caption{Two-layer neighborhood sampling around a target node (red) to approximate local industrial dependencies.}
	\label{fig:sampling} 
\end{figure} 

\subsection{Reverse graph convolution for realistic feature fabrication} 
\label{subsec32} 

SEGIA fabricates features for injected nodes to resemble \emph{plausible} industrial entities. In practice, node features often encode telemetry-derived vectors (e.g., traffic statistics, process measurements, device configuration signatures, event counts) or learned embeddings produced by upstream pipelines. Industrial monitoring systems commonly exhibit \emph{local similarity} (devices in the same cell/feeder share operating regimes; users in the same role share access patterns), which is exploited by both GNNs and homophily-aware defenses. SEGIA therefore synthesizes injected features by \emph{encoding local context} from the sampled neighborhood. Let $V_s^{k-1}$ and $V_s^{k}$ be sampled sets at consecutive layers and let $s^{k}=|V_s^{k}|$. Define the inter-layer connectivity matrix $M^k\in\mathbb{R}^{s^{k-1}\times s^{k}}$ as 
\begin{equation} 
M^k_{u,v} = 
\begin{cases} 1, & \text{if } A_{uv}=1,\\ 0, & \text{otherwise}, 
\end{cases} \qquad u \in V_s^{k-1},\; v \in V_s^{k}, \label{eq:M_industrial}
\end{equation} 
where $A$ is the adjacency matrix of $G$. We row-normalize 
\begin{equation} 
\widetilde M^k = \mathbf{Rownormal}(M^k), \label{eq:mean-M_industrial} 
\end{equation} 
which reduces sensitivity to degree heterogeneity common in industrial graphs (e.g., gateways/hubs vs.\ leaf sensors). Let $X^k\in\mathbb{R}^{s^{k}\times D}$ be the feature matrix of nodes in $V_s^{k}$ (with $X^K = X_{V_s^{K}}$). We then propagate information \emph{from outer layers to inner layers} via a reverse graph convolution: 
\begin{equation} X^{k-1} = \mathrm{ReLU}\!\left(\widetilde{M}^{k} X^{k} W^{k}\right), \qquad k = K,K-1,\dots,1, \label{eq:feat-agg_industrial}
\end{equation} 
where $W^{k}\in\mathbb{R}^{D\times D}$ are learnable parameters. After $K$ steps, $X^{0}\in\mathbb{R}^{s^{0}\times D}$ provides context-consistent feature vectors for injected nodes. 

\textbf{Engineering realism.} When features correspond to bounded physical/telemetry ranges (e.g., normalized measurements or standardized configuration fields), SEGIA can enforce realism by projecting $X^0$ into valid ranges (dataset- or system-specific bounds) during optimization. 

\subsection{Single-edge construction under industrial connectivity constraints} \label{subsec32b} Unlike many GIAs that rely on multiple injected edges, SEGIA enforces a strict \textbf{single-edge budget}: each injected node establishes exactly one link to the original graph. Concretely, for each injected node $u_i$, SEGIA attaches it to one selected anchor node $v_{j(i)}$ in the neighborhood of a target (often a target itself or a near neighbor), yielding an injection adjacency $A_I$ with exactly one nonzero entry per row. This low-footprint attachment reduces exposure to degree-based anomaly detection and aligns with OT segmentation and admission controls. 

\subsection{Optimization with a defense-aware surrogate} \label{subsec33} 

\subsubsection{Surrogate model (PrSGC) under edge pruning} 
\label{subsubsec331} 
In many deployments, the adversary does not know the victim model parameters and can only observe limited outcomes, motivating a surrogate. We adopt SGC for efficiency and introduce a pruned variant (PrSGC) to anticipate \emph{homophily-aware edge pruning}.

\paragraph{\bf Simple Graph Convolution (SGC)} 
\begin{equation*}
Z_{\mathrm{SGC}} = \operatorname{softmax}\!\left(\hat A^2 X W\right), \label{eq:sgc_industrial} 
\end{equation*} 
where $\hat{A}$ is the normalized adjacency and $W$ is trainable. 

\paragraph{\bf Pruned Simple Graph Convolution (PrSGC)} 
Define a binary pruning mask $P\in\{0,1\}^{N\times N}$: 
\begin{equation*} P_{uv} = \begin{cases} 1, & \text{if } \mathrm{sim}(x_u,x_v) \ge \varepsilon,\\ 0, & \text{otherwise}, \end{cases} \qquad u,v \in V, \label{eq:marker-P_industrial} 
\end{equation*} 
where $\mathrm{sim}(\cdot,\cdot)$ is cosine similarity and $\varepsilon$ is a pruning threshold. The PrSGC surrogate is 
\begin{equation} 
Z_{\mathrm{PrSGC}} = \operatorname{softmax}\!\left((\hat A \odot P)^2 X W\right). \label{eq:prsgc_industrial} 
\end{equation} 
Following common practice in gradient-based optimization (e.g.,~\cite{Nettack}), we use the linearized logits (i.e., ignoring the softmax normalization) when crafting attacks.

\subsubsection{Stealth-oriented objective via similarity regularization} 
\label{subsubsec332} Industrial monitoring pipelines may penalize entities whose attributes are inconsistent with their local context. To improve stealthiness, SEGIA encourages injected nodes to remain similar to their anchor neighborhoods, reducing detectability under homophily-based sanitization. We recall node-centric homophily~\cite{HAO}. 
\begin{definition}[Node-centric homophily] 
	Let $x_u$ denote the feature vector of node $u$. The homophily of node $u$ is defined as
	\[ h_u = \mathrm{sim}(r_u, x_u),\qquad r_u = \sum_{j\in \mathcal N(u)} \frac{1}{\sqrt{d_j}\sqrt{d_u}}\, x_j, \] 
	where $d_u$ is the degree of node $u$ and $\mathcal N(u)$ is its neighbor set. 
\end{definition} 

\subsection{Attack objective and constraints} 
\label{subsec:attack-objective} 
Let $f_{\theta}$ be the victim GNN with parameters $\theta^*$ trained on the clean graph $G=(A,X)$ by minimizing ${\cal L}_{\mathrm{train}}$ on labeled nodes. Given targets $V_t$, define the target-set loss 
\begin{equation*} {\cal L}_{\mathrm{tgt}}(G',V_t)=\sum_{u\in V_t}\ell\!\left(f_{\theta^*}(G')_u,\,Y_u\right), \label{eq:Ltgt} 
\end{equation*} 
and define the base injection loss
\begin{equation} {\cal L}_{\mathrm{GIA}}(G') := -\,{\cal L}_{\mathrm{tgt}}(G',V_t), \label{eq:gia-loss_industrial} 
\end{equation} 
such that minimizing ${\cal L}_{\mathrm{GIA}}$ increases the classification error on targets. To improve stealthiness, we add similarity regularization between each injected node $u_i$ and its anchor $v_{j(i)}$: 
\begin{equation} 
{\cal L}_{\mathrm{atk}}(G') = {\cal L}_{\mathrm{GIA}}(G') - \alpha\sum_{i=1}^{N_I}\mathrm{sim}(x_{u_i},x_{v_{j(i)}}), \qquad \alpha>0. \label{eq:atk-loss_industrial}
\end{equation} 
Formally, the attacker solves 
\begin{align} 
\min_{G'}\quad & {\cal L}_{\mathrm{atk}}(G') \notag\\ \text{s.t.}\quad 
& \theta^* = \arg\min_\theta {\cal L}_{\mathrm{train}}\!\left(f_\theta(G_{\mathrm{train}})\right), \notag\\ 
& \|G' - G\| \le \Delta, 
\label{eq:atk-opt_industrial} 
\end{align} 
where $\Delta$ enforces \textbf{one added edge per injected node} (single-edge budget) and may additionally include feature-range constraints to preserve telemetry realism. In SEGIA, the surrogate (PrSGC) approximates the defended decision boundary, while the single-edge constraint is enforced by construction of $A_I$.

\subsection{Theoretical analysis (security interpretation)} 
\label{subsec:theory} 
Homophily-aware defenses remove edges that connect dissimilar entities (e.g., inconsistent telemetry versus neighborhood baselines). Let $g_\psi$ denote a defender that prunes edges $(u,v)$ when $\mathrm{sim}(x_u,x_v)<\varepsilon$, producing a defended graph $g_\psi(G')$. Let $h_v(G)$ be node-centric homophily on $G$, and define the multiset \[ H_G := \{ h_v(G) : v \in V \}. \] Let $\operatorname{dis}(H_G,H_{G'})$ denote a distance between homophily-score distributions (e.g., total variation or Wasserstein). Theorem~\ref{thm1} formalizes that SEGIA reduces homophily disruption compared with conventional GIAs and yields stronger post-pruning attack effectiveness under the stated assumptions. 

\begin{theorem} \label{thm1} 
	Let $G$ be an undirected connected graph without isolated nodes. Assume each class $c\in\{1,\dots,C\}$ has at least one labeled node. Let $f_{\theta}$ be a linearized GNN surrogate (e.g., PrSGC) trained on $G$, and let ${\cal L}_{\mathrm{GIA}}$ denote the base injection loss. Let $G^{\mathrm{GIA}}$ and $G^{\mathrm{SEGIA}}$ be attacked graphs produced by a conventional GIA and by SEGIA, respectively. For the similarity-regularized objective in~\eqref{eq:atk-loss_industrial}, we have 
	\begin{align} 
	\operatorname{dis}\!\big(H_G, H_{G^{\mathrm{SEGIA}}}\big) &\le \operatorname{dis}\!\big(H_G, H_{G^{\mathrm{GIA}}}\big), \label{eq:homophilyineq}\\ {\cal L}_{\mathrm{atk}}\!\big(g_\psi(G^{\mathrm{SEGIA}})\big) &\le {\cal L}_{\mathrm{atk}}\!\big(g_\psi(G^{\mathrm{GIA}})\big). 
	\label{eq:lossineq} 
	\end{align} 
\end{theorem} 

{\bf Remark.} 
Theorem~\ref{thm1} indicates that SEGIA perturbs neighborhood similarity patterns less than conventional GIAs and thus better survives homophily-based pruning. In industrial monitoring, this corresponds to a counterfeit entity whose telemetry remains locally plausible while still causing erroneous GNN outputs. 

\subsection{Comparison with existing methods (engineering perspective)}

Many injection attacks (e.g., TDGIA, G-NIA) do not explicitly account for stealth under homophily-oriented sanitization, and camouflage-oriented methods (e.g., HAO, CANA) often rely on multiple injected edges, increasing structural footprint. SEGIA targets the low-footprint regime by enforcing a \textbf{single-edge} attachment per injected entity, reducing edge-budget cost and limiting degree anomalies while maintaining strong attack impact under defenses (Fig.~\ref{comparison} and Table~\ref{cana}). 
\begin{figure*}[!tb]
	\footnotesize
	\centering
	\subfigure[Original Graph]{\includegraphics[width=0.2\textwidth]{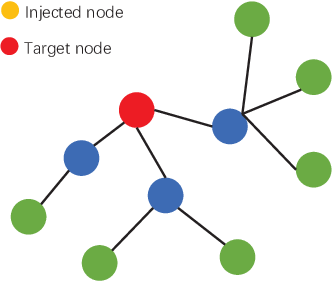}}
	\mbox{}	
	\subfigure[General Injection]{\includegraphics[width=0.2\textwidth]{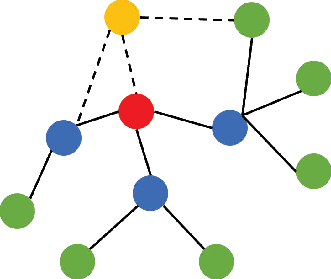}}
	\mbox{}	
	\subfigure[Our Injection]{\includegraphics[width=0.2\textwidth]{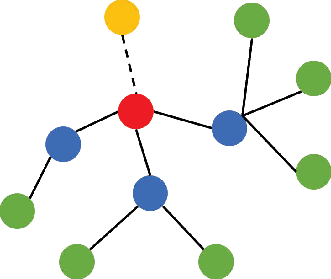}}
	
	\caption{Comparison of our method with other approaches. (a) Original  graph; (b) General injection (multiple edge budgets); (c) Our method (only one edge budget).}
	\label{comparison}
\end{figure*}

\begin{table*}[!htb]
	\small
	\caption{Comparison of the total edge budget.} \label{cana}
	\centering
	\begin{tabular}{@{}lllll@{}}
		\hline
		Datasets &  Node  budget & Feature  range& Edge  budget (ours) & Edge budget (CANA)\\
		\hline
		ogbn-products & 2,099 & [-20, 20] & 2,099 & 6,297 \\
		reddit & 2,001 & [-20, 20] & 2,001 & 14,007 \\
		ogbn-arxiv & 33,869 & [-1, 1] & 33,869 & 474,166 \\
		\hline
	\end{tabular}
\end{table*}

\begin{table*}[!htb]
	\footnotesize
	\caption{Summary of graph adversarial attacks.}
	\label{comparatack}
	\centering
	\begin{tabular}{@{}llllll@{}}
		\hline
		Methods & Attack knowledge& Attack type& Attack setting & Perturbation type  & Victim model\\
		\hline
		PGD & White-box  &Untargeted  &Evasion and Poisoning   &Add/Delete edge   & GNN\\
		NIPA & Gray-box  &Targeted   &Poisoning  & Inject nodes   & GNN\\
		TDGIA  &Black-box  &Targeted  &Evasion  & Inject nodes  & GNN\\
		G-NIA  &Black-box  &Targeted  &Evasion  & Inject nodes  & GNN\\
		\hline
		{\bf SEGIA (ours)} & Black-box  &Targeted  &Evasion  & Inject nodes  & GNN\\
		\hline
	\end{tabular}
\end{table*}

\subsection{Computational complexity} 

Algorithm~\ref{alg:SEGIA} comprises neighborhood sampling and training the feature generator. Let $|V_t|$ be the number of targets, $K$ the sampling depth, $m$ the number of sampled neighbors per layer, and $T$ the number of iterations. The sampling complexity is $O(|V_t|\cdot m^K)$. The feature-generator training cost depends on sampled sizes and feature dimension; in particular, the dominant matrix multiplications in~\eqref{eq:feat-agg_industrial} scale with $\sum_{k=1}^{K} O\!\left(T\cdot \mathrm{nnz}(\widetilde M^k)\cdot D\right)$ (and $D\times D$ transforms), which remains practical when sampling bounds $\mathrm{nnz}(\widetilde M^k)$ by construction.

We next evaluate SEGIA against representative injection baselines under both undefended and defended pipelines, with budgets chosen to isolate performance in the single-edge regime.

\section{Experiments} 
\label{expe} 
We evaluate SEGIA against representative node injection baselines on GNN classifiers \emph{with} and \emph{without} defense mechanisms. For defended settings, we compare against PGD~\cite{PGD}, TDGIA~\cite{TDGIA}, and G-NIA~\cite{G-NIA}, as well as defense-augmented variants that incorporate HAO~\cite{HAO} and CANA~\cite{CANA} (denoted +HAO and +CANA). For undefended settings, we follow~\cite{QUAGIA} and compare against TDGIA, ATDGIA~\cite{ATDGIA}, AGIA~\cite{ATDGIA}, G2A2C~\cite{G2A2C}, and QUAGIA~\cite{QUAGIA}. \textbf{Budgets and constraints.} We match the perturbation rate $Pr$ (injected-node ratio) across methods. SEGIA additionally enforces the deployment-oriented \textbf{single-edge constraint} by construction, i.e., each injected node attaches through exactly one edge, so the number of injected edges equals the number of injected nodes. Baselines may use multi-edge injections depending on their original formulations; when discussing results, we therefore distinguish \emph{node budget} (matched by $Pr$) from \emph{edge footprint} (method-dependent). 

\subsection{Settings} 
\label{subsec41} 
\subsubsection{Datasets} 
\label{subsubsec411} 
We conduct experiments on eight datasets (Table \ref{tab:dataset}). For defended settings, we report results on \textsc{Reddit}~\cite{G-NIA}, \textsc{ogbn-products}, and \textsc{ogbn-arxiv}~\cite{hu2020open}. To characterize limitations under dense connectivity, we additionally evaluate on \textsc{Amazon Computers} and \textsc{Amazon Photo}. All experiments are performed on the largest connected component (LCC), following~\cite{Nettack,zugner2018adversarial}. For undefended settings, we follow~\cite{QUAGIA} and evaluate on \textsc{Cora}, \textsc{PubMed}, and \textsc{grb-cora}.

\subsection{Experimental Analysis}
\label{subsec45} 

\subsubsection{Attack Performance and System-Level Risk} 
\label{subsubsec451} 
Tables~\ref{tab:experiment} and~\ref{tab:undefended-experiment} quantify SEGIA under defended and undefended pipelines, respectively. In industrial graph-based monitoring, target-node misclassification can manifest as missed alarms, incorrect asset state labels, or unsafe prioritization. Conventional injection baselines (e.g., PGD/TDGIA/G-NIA) show limited or inconsistent impact across defenses, whereas SEGIA yields the strongest disruption across datasets and defense models. These results indicate a deployment-relevant risk: \emph{low-footprint counterfeit entities (single-edge per injected node) can remain influential after sanitization and propagate errors through message passing.} Under undefended multi-model evaluation (GCN, APPNP~\cite{APPNP}, and GAT~\cite{GAT}), SEGIA achieves strong disruption at low perturbation rates ($Pr=1\%$ and $Pr=3\%$), improving over QUAGIA by $8.83\%$ and $4.46\%$ on average across datasets. At $Pr=5\%$, QUAGIA can be competitive on some datasets, which is consistent with an operational trade-off: multi-edge injection schemes can increase influence by creating additional propagation paths but also expand structural footprint. In contrast, SEGIA targets the constrained-admission regime by enforcing a single-edge attachment per injected entity.
\begin{table*}[!htbp]
	\caption{LCC denotes the largest connected component of the graph.  $N_{LCC}$  and  $E_{LCC}$ represent the number of nodes and edges in the LCC, respectively. Avg. is abbreviation for Average.} \label{tab:dataset}
	\centering
	\begin{tabular}{@{}ccccccc@{}}
		\hline
		Datasets &  $N_{LCC}$ & $E_{LCC}$ & Classes & Avg.  Degree& Feature  Dimension& Feature  Range\\
		\hline
		ogbn-products &  10,494 & 38,872 & 35 & 3.70 & 100 & [-74.70, 152.71]\\
		reddit &  10,004 & 73,512 & 41 & 7.35 & 602 & [-22.89, 80.85]\\
		ogbn-arxiv &  169,343 & 2,484,941 & 39 & 14.67 & 128 & [-1.39, 1.64]\\
		Computers &  13,752 & 491,722 & 10 & 35.76 & 767 & [0, 1]\\
		Photo &  7,650 & 238,162 & 8 & 31.13 & 745 & [0, 1]\\
		\hline
	\end{tabular}
\end{table*}

\begin{table}[!htbp]
	\caption{Misclassification rate (\%) of the proposed SEGIA across  defense methods FLAG and GNNGuard on ogbn-products, reddit and ogbn-arxiv datasets.} \label{tab:experiment}
	\centering
	% \begin{tabular}{@{}l|lll@{}}
	\begin{tabular}{lccc}
		\hline
		Datasets    & Attack methods & FLAG  & GNNGuard \\
		\hline
		\multirow{11}{*}{ogbn-products}     & Clean & 31.11 & 20.10 \\
		& PGD   & 29.87 & 33.21 \\
		& \quad +HAO  & 27.06 & 33.49 \\
		& \quad +CANA & 31.92 & 35.11 \\
		
		& TDGIA & 33.35 & 39.21 \\
		& \quad +HAO  & 25.58 & 33.68 \\
		& \quad +CANA & 27.49 & 30.63 \\
		
		& G-NIA & 32.92 & 32.25 \\
		& \quad +HAO  & 21.06 & 21.92 \\
		& \quad +CANA & 40.11 & 34.21 \\
		
		& {\bf SEGIA (ours)}  & \textbf{79.13} & \textbf{75.85} \\
		\hline
		\multirow{11}{*}{reddit}            & Clean & 12.79 & 18.39 \\
		& PGD   & 17.99 & 24.79 \\
		& \quad +HAO  & 15.54 & 52.22 \\
		& \quad +CANA & 23.44 & 34.78 \\
		& TDGIA & 20.09 & 26.34 \\
		& \quad +HAO  & 15.79 & 50.47 \\
		& \quad +CANA & 16.89 & 25.84 \\
		
		& G-NIA & 22.99 & 17.99 \\
		& \quad +HAO  & 11.19 & 14.79 \\
		& \quad +CANA & 36.38 & 24.04 \\
		
		&{\bf SEGIA (ours)}& \textbf{79.06} & \textbf{75.66} \\
		\hline
		\multirow{11}{*}{ogbn-arxiv}        & Clean & 35.63& 29.09 \\
		& PGD   & 54.46 & 31.64 \\
		& \quad +HAO  & 52.54 & 40.68 \\
		& \quad +CANA & 54.16 & 50.57 \\
		& TDGIA & 51.82 & 32.37 \\
		& \quad +HAO  & 51.99 & 36.04 \\
		& \quad +CANA & 51.25 & 30.20 \\
		& G-NIA & 52.94 & 28.43 \\
		& \quad +HAO  & 38.25 & 31.45 \\
		& \quad +CANA & 44.16 & 32.51 \\
		& {\bf SEGIA (ours)} & \textbf{84.40} & \textbf{73.45} \\
		\hline
	\end{tabular}
\end{table}
\begin{table*}[!htbp]
	\caption{Comparison of average classification accuracy, where lower values indicate better attack performance and $Pr$ is the perturbation rate.} \label{tab:undefended-experiment}
	\centering
	\begin{tabular}{lccccccc}
		\hline
		Datasets & $Pr$ & TDGIA & ATDGIA & AGIA & G2A2C & QUAGIA & SEGIA \\
		\hline
		\multirow{3}{*}{Cora}     
		& 0.01 & 0.988 & 0.973 & 0.974 & 0.988 & 0.956 & \textbf{0.877} \\
		& 0.03 & 0.946 & 0.924 & 0.947 & 0.965 & 0.881 & \textbf{0.834} \\
		& 0.05 & 0.899 & 0.898 & 0.920 & 0.953 & \textbf{0.815} & 0.818 \\
		
		\hline
		\multirow{3}{*}{PubMed}            
		& 0.01 & 0.990 & 0.984 & 0.982 & 0.986 & 0.955 & \textbf{0.853} \\
		& 0.03 & 0.970 & 0.955 & 0.951 & 0.979 & 0.879 & \textbf{0.795} \\
		& 0.05 & 0.950 & 0.927 & 0.924 & 0.974 & 0.820 & \textbf{0.728} \\
		
		\hline
		\multirow{3}{*}{grb-cora}       
		& 0.01 & 0.975 & 0.967 & 0.969 & 0.991 & 0.955 & \textbf{0.871} \\
		& 0.03 & 0.910 & 0.904 & 0.908 & 0.936 & 0.848 & \textbf{0.845} \\
		& 0.05 & 0.874 & 0.849 & 0.860 & 0.897 & \textbf{0.765} & 0.807 \\
		\hline
	\end{tabular}
\end{table*}

\begin{figure}[!htb]
	\centering
	\includegraphics[width=1\linewidth]{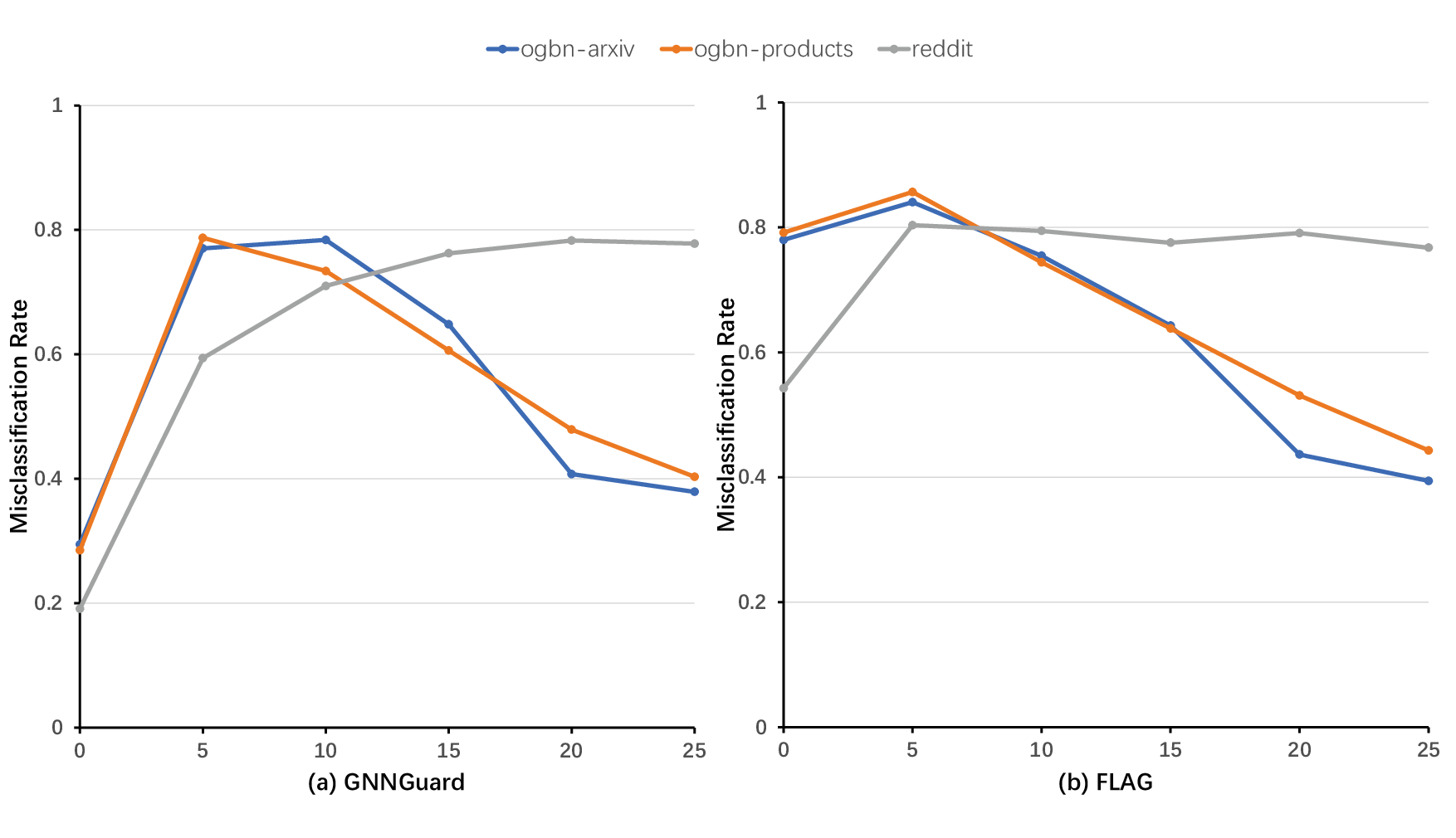}
	\caption{The influence of the  hyperparameter $\alpha$ on the  misclassification rate.}
	\label{fig:alpha}
\end{figure}
\subsubsection{Sensitivity Analysis}
\label{subsubsec:sensitivity} 

Fig.~\ref{fig:alpha} shows that moderate $\alpha$ yields the best disruption, while larger $\alpha$ emphasizes similarity regularization and reduces attack aggressiveness. Fig.~\ref{fig:k-hop} indicates that $K=2$ provides most gains under GNNGuard, with diminishing returns thereafter. These trends align with localized dependency structure in industrial graphs, where short-range neighborhoods dominate both message passing influence and practical consistency monitoring. 

\begin{figure}[!htb]
	\centering
	\includegraphics[width=1\linewidth]{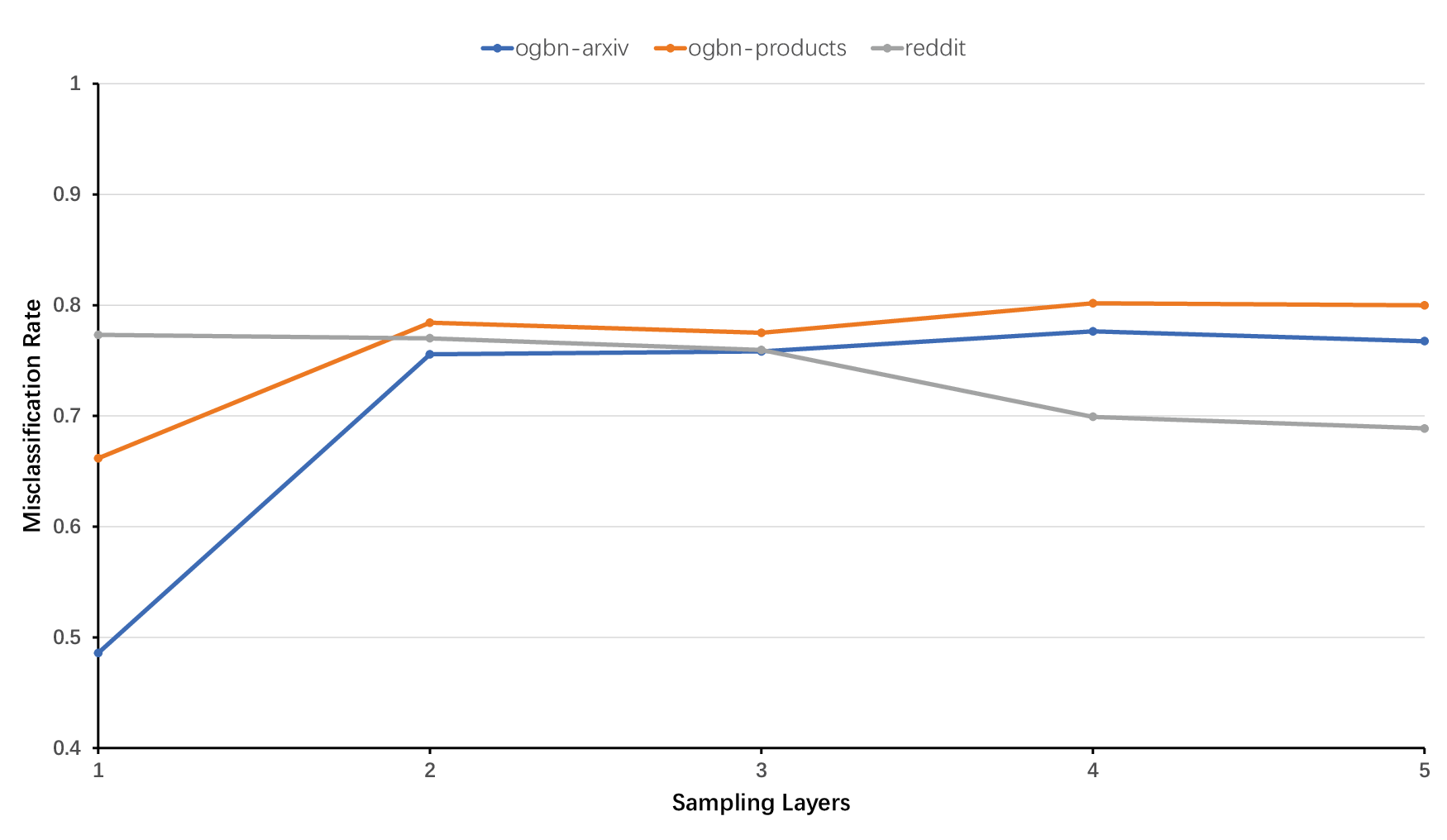}
	\caption{The effect of parameter $K$ on the misclassification rate for all the three datasets by using defense model GNNGuard.}
	\label{fig:k-hop}
\end{figure}
\subsubsection{Ablation Studies: Security Engineering Justification}
\label{subsubsec:ablation} 
\begin{figure}[!htbp]
	\centering
	\includegraphics[width=1\linewidth]{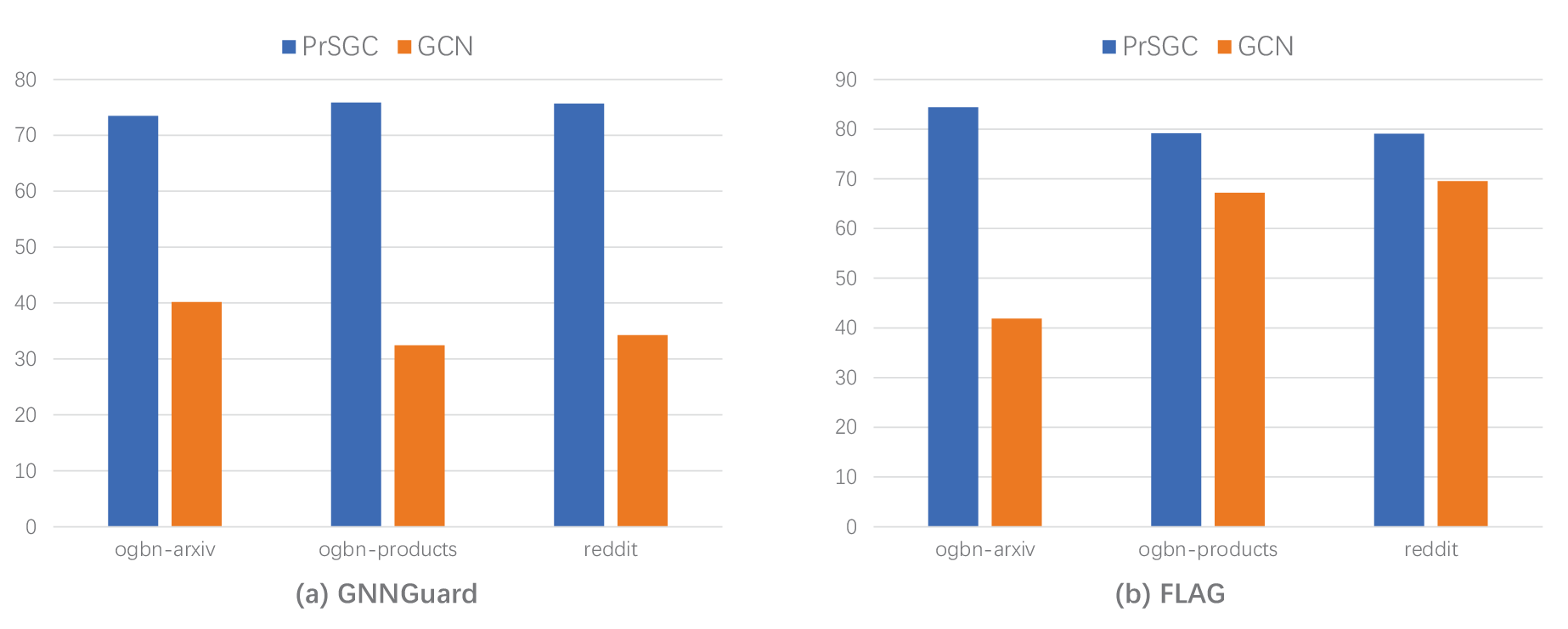}
	\caption{Comparison of misclassification rate (\%) using surrogate models PrSGC and GCN.}
	\label{fig:ablation_expt_surrogate}
\end{figure}

\begin{table}[!htb]
	\caption{Comparison of misclassification rate (MR) and runtime with and without sampling on ogbn-arxiv dataset.} \label{tab:ablation_expt_sampling}
	\centering  
	% \begin{tabular}{@{}l|lll@{}}
	\begin{tabular}{c|cc|cc}
		\hline
		\multirow{2}{*}{Defense model}    & \multicolumn{2}{c|}{Sampling} & \multicolumn{2}{c}{Without sampling} \\
		& MR(\%) & runtime(s) & MR(\%) & runtime (s) \\
		\hline
		GNNGuard & 73.45 & 675.727 & 76.26 & 2390.177 \\
		\hline
		FLAG & 84.40 & 585.094 & 85.76 & 2394.147 \\
		\hline
	\end{tabular}
\end{table}
\begin{table}[!htbp]
	\caption{Misclassification rate (\%) of the proposed SEGIA on Computers and Photo datasets.} \label{tab:shortcoming experiment}
	\centering 
	\begin{tabular}{lrr}
		\hline
		Datasets &  FLAG & GNNGuard\\
		\hline
		Computers &  29.70 & 8.36\\
		Photo &  17.71 & 6.34\\
		\hline
	\end{tabular}
\end{table}

\textbf{PrSGC surrogate.} 
Using PrSGC improves attack success under GNNGuard and FLAG (Fig.~\ref{fig:ablation_expt_surrogate}), indicating that explicitly modeling edge pruning is important for persistence: injected features and attachments optimized for the post-sanitization graph are more likely to survive pruning and remain disruptive. 

\textbf{Neighborhood sampling.} Sampling reduces runtime by $\sim3.5\times$--$4\times$ with minor performance loss (Table~\ref{tab:ablation_expt_sampling}), supporting feasibility under partial observability and bounded compute. From a defense standpoint, scale alone does not provide protection; localized, efficient attacks can still target critical nodes. 

\subsubsection{Limitations}
\label{subsubsec:limitations} 
SEGIA is less effective on high-degree graphs (Table~\ref{tab:shortcoming experiment}) because a single injected edge contributes a smaller fraction of the target neighborhood, attenuating influence. This limitation is most relevant to densely connected IT-side graphs; however, many OT/IIoT graphs are sparse by design due to segmentation, where single-edge admissions remain plausible and impactful.

\section{Conclusion}
\label{conl}
This paper studies black-box graph injection attacks as a deployment-oriented threat to IIoT/CPS graph-based monitoring. We propose SEGIA, a single-edge node injection framework that combines multi-hop neighborhood sampling, reverse message-passing feature synthesis, and a pruning-aware surrogate with similarity regularization. Experiments across datasets and defenses show that SEGIA achieves at least $25\%$ higher attack success than strong baselines under substantially smaller edge footprints, indicating that minimal, plausible node admissions can compromise GNN-driven monitoring reliability and enable risk propagation through message passing. These findings emphasize a trustworthiness perspective for industrial AI: graph construction and node admission are security boundaries, and homophily-oriented pruning alone may not prevent feature-consistent injections.

Deployment-grade protection requires (i) provenance-aware admission validation for devices/identities/assets, (ii) neighborhood-consistency monitoring that incorporates multi-modal evidence (telemetry, inventory, and communication constraints), and (iii) robustness mechanisms that preserve accuracy without excessive over-smoothing when increasing propagation depth. Extending injection threats and defenses to heterogeneous and dynamic industrial graphs remains an important direction.

\bibliographystyle{IEEEtran}
\bibliography{references}

\section*{Appendix: Proof of Theorem~\ref{thm1}}

This appendix provides the proof of Theorem~\ref{thm1}. We consider a linear message-passing surrogate that captures deployment-oriented sanitization via edge pruning (e.g., PrSGC). For a two-layer linearized surrogate, the hidden representation is
\begin{equation}
\label{eq:app_hidden}
H^{(2)} = (\hat A \odot P)^2 X W,
\end{equation}
where $\hat A$ is the normalized adjacency, $P$ is the binary pruning mask, $\odot$ denotes Hadamard multiplication, $X$ is the feature matrix, and $W$ is the trainable weight matrix. The Jacobian of $H^{(2)}_v$ with respect to $X_u$ follows directly from~\eqref{eq:app_hidden}:
\begin{equation}
\label{eq:app_jacobian}
\frac{\partial H^{(2)}_v}{\partial X_u} = \big[(\hat A \odot P)^2\big]_{vu}\, W.
\end{equation}

\begin{proof}
	We first establish the claim for binary classification and then outline the extension to the multi-class case.
	
	\noindent\textbf{Binary case.}
	Under the feature encoding in Theorem~\ref{thm1}, each node feature is
	\[
	X_u =
	\begin{cases}
	[1,-1]^\top, & Y_u = 0,\\[2pt]
	[-1,1]^\top, & Y_u = 1,
	\end{cases}
	\]
	where $Y_u\in\{0,1\}$ is the class label. Let ${\cal L}_{\mathrm{GIA}}$ denote the base graph injection attack loss and define the similarity-regularized objective
	\begin{equation}
	\label{eq:app_atk}
	{\cal L}_{\mathrm{atk}} = {\cal L}_{\mathrm{GIA}} - \alpha\,\mathrm{sim}(x_{u_i},x_{v_j}),\qquad \alpha>0,
	\end{equation}
	where $u_i$ is an injected node, $v_j$ is its (single) anchor neighbor in the original graph, and $\mathrm{sim}(\cdot,\cdot)$ denotes cosine similarity. For a target node $v$, let ${\cal L}_v$ denote its contribution to the overall loss. Using~\eqref{eq:app_jacobian}, the gradient of ${\cal L}_v$ with respect to an injected feature vector $X_u$ is
	\begin{equation*}
	\label{eq:app_grad}
	\frac{\partial {\cal L}_v}{\partial X_u} = \frac{\partial {\cal L}_v}{\partial H^{(2)}_v}\, \big[(\hat A \odot P)^2\big]_{vu}\, W.
	\end{equation*}
	For binary cross-entropy with one-hot encoding, $\partial {\cal L}_v / \partial H^{(2)}_v$ is proportional to a signed class-separating direction; without loss of generality, we write this direction as $[-1,\,1]^\top$. The regularization term in~\eqref{eq:app_atk} contributes a component that increases $\mathrm{sim}(x_{u_i},x_{v_j})$, i.e., it drives $x_{u_i}$ toward $x_{v_j}$ and thus reduces the angular deviation between injected and anchor features. Consider a projected gradient descent update on $X_u$ with step size $\delta>0$:
	\begin{equation}
	\label{eq:app_pgd}
\resizebox{.94\linewidth}{!}{$	X_u^{(t+1)} = X_u^{(t)} + \delta\,\mathrm{sgn}\!\left( \big[(\hat A \odot P)^2\big]_{vu}[-1,\,1]^\top - \alpha[1,\,-1]^\top \right)W,
	$}
	\end{equation}
	where $\mathrm{sgn}(\cdot)$ is applied element-wise. Assuming the entries of $W$ are nonnegative (or, equivalently, absorbing signs into the class-separating direction), the update direction in~\eqref{eq:app_pgd} decreases as the similarity term counterbalances the attack gradient. At a stationary point (or when updates stall under projection/constraints), the effective update direction satisfies
	\begin{equation*}
	\label{eq:app_balance}
	\big[(\hat A \odot P)^2\big]_{vu}[-1,\,1]^\top - \alpha[1,\,-1]^\top = \mathbf 0,
	\end{equation*}
	which corresponds to an operating point where the injected feature is aligned with the anchor neighborhood while still exerting maximal influence through the defended message-passing operator $(\hat A \odot P)^2$. This balance enforces a \emph{stealth constraint}: the injected node remains locally homophilous and is therefore less likely to be removed by similarity-based sanitization, consistent with the ``homophily-unnoticeable'' configuration analyzed in~\cite{HAO}. Following the argument in~\cite{HAO}, classical multi-edge GIAs that do not explicitly regularize similarity tend to reduce node-centric homophily more than similarity-constrained injection. Consequently, for any node $v$,
	\begin{equation}
	\label{eq:app_homo_node}
	h^{\mathrm{GIA}}_v \le h^{\mathrm{SEGIA}}_v.
	\end{equation}
	Aggregating~\eqref{eq:app_homo_node} over $v\in V$ implies the distributional inequality in~\eqref{eq:homophilyineq}:
	\[
	\operatorname{dis}\!\big(H_G, H_{G^{\mathrm{SEGIA}}}\big) \le \operatorname{dis}\!\big(H_G, H_{G^{\mathrm{GIA}}}\big).
	\]
	Finally, let $g_\psi$ denote a homophily-based defender that prunes edges $(u,v)$ when $\mathrm{sim}(x_u,x_v)<\varepsilon$. Since SEGIA explicitly increases similarity between each injected node and its anchor, its injected edges are more likely to satisfy the pruning criterion and remain in the defended graph. Therefore, after applying $g_\psi$, SEGIA preserves more effective attack influence than a baseline GIA under the same budget, yielding the defended-loss inequality in~\eqref{eq:lossineq}:
	\[
	{\cal L}_{\mathrm{atk}}\!\big(g_\psi(G^{\mathrm{SEGIA}})\big) \le {\cal L}_{\mathrm{atk}}\!\big(g_\psi(G^{\mathrm{GIA}})\big).
	\]
	This completes the proof for the binary case.
	
	\noindent\textbf{Multi-class extension (outline).}
	Under the simplex-style encoding in Theorem~\ref{thm1}, projecting the multi-class cross-entropy gradient onto any two-class subspace recovers the binary update structure. The similarity term enforces local alignment in each projection, so the above homophily and defended-loss inequalities extend to $C>2$.
\end{proof}
\end{document}